\documentclass[conference]{IEEEtran}
\IEEEoverridecommandlockouts
\usepackage{cite}
\usepackage{amsmath,amssymb,amsfonts}
\usepackage{algorithmic}
\usepackage{graphicx}
\usepackage{textcomp}
\usepackage{xcolor}
\def\BibTeX{{\rm B\kern-.05em{\sc i\kern-.025em b}\kern-.08em
    T\kern-.1667em\lower.7ex\hbox{E}\kern-.125emX}}

\usepackage{bm}
\usepackage{subfigure}
\usepackage{multirow}
\usepackage{array}
\newcolumntype{P}[1]{>{\centering\arraybackslash}p{#1}}
\usepackage{hyperref}
\usepackage{comment}
\usepackage{multicol}
\usepackage{tabularx}
\usepackage{adjustbox}

\usepackage{etoolbox}
\makeatletter
\patchcmd{\@makecaption}
  {\scshape}
  {}
  {}
  {}
\makeatletter
\patchcmd{\@makecaption}
  {\\}
  {.\ }
  {}
  {}
\makeatother
\def\tablename{Table}

\usepackage{caption}
\newtheorem{theorem}{Theorem}
\newtheorem{assumption}{Assumption}
\newtheorem{proposition}{Proposition}

\makeatletter
\DeclareRobustCommand{\qed}{%
  \ifmmode 
  \else \leavevmode\unskip\penalty9999 \hbox{}\nobreak\hfill
  \fi
  \quad\hbox{\qedsymbol}}
\newcommand{\openbox}{\leavevmode
  \hbox to.77778em{%
  \hfil\vrule
  \vbox to.675em{\hrule width.6em\vfil\hrule}%
  \vrule\hfil}}
\newcommand{\qedsymbol}{\openbox}
\newenvironment{proof}[1][\proofname]{\par
  \normalfont
  \topsep6\p@\@plus6\p@ \trivlist
  \item[\hskip\labelsep\itshape
    #1.]\ignorespaces
}{%
  \qed\endtrivlist
}
\newcommand{\proofname}{Proof}
\makeatother

\newcommand{\indep}{\perp \!\!\! \perp}

\begin{document}

\title{To Predict or to Reject: Causal Effect Estimation with Uncertainty on Networked Data}


\author{\IEEEauthorblockN{Hechuan Wen$^{1}$, Tong Chen$^{1*}$, Li Kheng Chai$^{2}$, Shazia Sadiq$^{1}$, Kai Zheng$^{3}$, Hongzhi Yin$^{1}$}
\IEEEauthorblockA{\textit{$^{1}$The University of Queensland, Australia} \\
\textit{$^{2}$Health and Wellbeing Queensland, Australia}\\
\textit{$^{3}$University of Electronic Science and Technology of China, China}\\
\{h.wen, tong.chen, h.yin1\}@uq.edu.au, likheng.chai@hw.qld.gov.au, shazia@eecs.uq.edu.au, zhengkai@uestc.edu.cn}
\thanks{$^*$Tong Chen is the corresponding author.}
}

\maketitle

\begin{abstract}

Due to the imbalanced nature of networked observational data, the causal effect predictions for some individuals can severely violate the positivity/overlap assumption, rendering unreliable estimations. Nevertheless, this potential risk of individual-level treatment effect estimation on networked data has been largely under-explored. To create a more trustworthy causal effect estimator, we propose the uncertainty-aware graph deep kernel learning (GraphDKL) framework with Lipschitz constraint to model the prediction uncertainty with Gaussian process and identify unreliable estimations. To the best of our knowledge, GraphDKL is the first framework to tackle the violation of positivity assumption when performing causal effect estimation with graphs. With extensive experiments, we demonstrate the superiority of our proposed method in uncertainty-aware causal effect estimation on networked data. The code of GraphDKL is available at \url{https://github.com/uqhwen2/GraphDKL}.

\end{abstract}


\section{Introduction \label{section:intro}}
Estimating causal effect to support decision-making in high-stake domains such as healthcare, education, and e-commerce is crucial. With the prevalence of networked data, \cite{guo2020learning} has recently started exploring both the features of individuals (i.e., nodes) and their structural connectivity (i.e., edges) with graph neural networks (GNNs) for causal effect estimation. 

Owing to the inherent nature of observational data, violation of positivity is inevitable yet potentially devastating for causal effect estimation at the individual level, as the low-confidence predictions on non-overlapping samples may suggest a wrong treatment or introduce false hope \cite{jesson2021quantifying}. For networked data where individuals are mutually connected, the violation of positivity is further amplified because of the presence of additional structural information. As shown in Figure \ref{fig:overlap} (a), to predict the health status of older users (i.e., control group) after using dietary supplements, one may train a causal estimator based on observational data from younger users (i.e., treated groups). 
 However, as Figure \ref{fig:overlap} (b) depicts, although decent counterfactual estimations can be made within the overlapping area, a higher risk exists when estimating in the non-overlapping area of a different group. In worse scenarios, the predicted treatment outcomes contradict the ground truth, leading to a false recommendation with adverse effect. 

\begin{figure}[t!]
  \centering

  \subfigure[]{\includegraphics[width=0.24\textwidth]{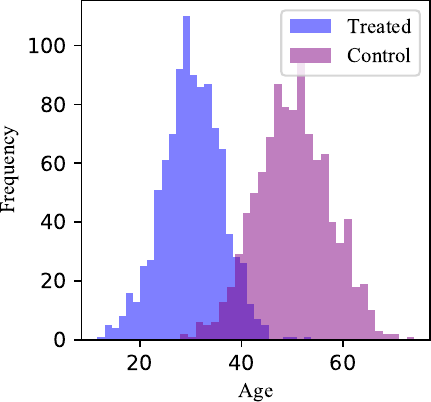}}
  \subfigure[]{\includegraphics[width=0.23\textwidth]{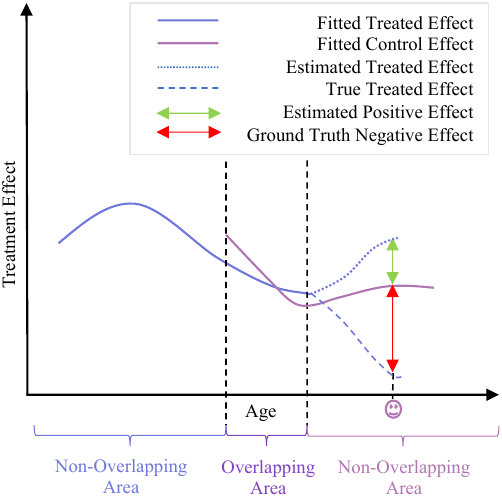}}
\vspace{-0.4cm}
  \caption{(a) Histogram of two treatment groups on a one-dimensional toy dataset w.r.t. age. (b) The high risk of causal effect estimation in the non-overlapping area due to violation of positivity.}
  \label{fig:overlap}
\vspace{-0.5cm}
\end{figure}

Thus, instead of blindly making recommendations based on low-confidence predictions on individual treatment effect, a more desirable capability of a causal estimator is to flag every highly uncertain estimation resulted from violation of positivity, which can be deferred for human inspections and used to guide improvements on the observational data collection process.
However, existing solutions for measuring the uncertainty of each counterfactual prediction \cite{jesson2020identifying} are predominantly centered around tabular data without any inter-dependencies among samples. This renders existing uncertainty-aware methods unable to capture the nuanced divergence between samples in graph-structured data, given the combinatorial impact from not only individuals' own variables but also their connections with others in the network.

To fill the gap in uncertainty-aware causal estimation with networked data in the presence of positivity violation, we propose our Graph Deep Kernel Learning (GraphDKL) framework which offers uncertainty estimation to identify the likely unreliable counterfactual predictions. We introduce Gaussian process (GP) to the GNN architecture, so as to let the causal effect estimator benefit from the probabilistic nature of GP by referring to the derived prediction variance as a precise indicator of the estimator's confidence in each prediction it makes. To increase the scalability of GraphDKL on large graphs, we further design a sparse variational optimization process to replace the time-consuming covariate matrix inversion in GP with a more computationally tractable learning objective, which significantly reduces the complexity from $\mathcal{O}(N^{3})$ to $\mathcal{O}(M^{2}N)$ with $N$ being the number of training samples ($M\ll N$). Meanwhile, another notable obstacle with deep architectures used for causal effect estimation is the feature collapse issue \cite{van2021feature}, i.e., two distinct raw data points can share nearly identical representations after being mapped to the latent space via deep layers. Despite the richer information embedded after the graph convolution, the collapse of different individuals' latent representations can seriously hinder the uncertainty quantification. For instance, a sample from the non-overlapping area is intuitively associated with a stronger uncertainty in its counterfactual prediction. However, in the networked data, if it is connected to one or more samples from the overlapping area, then its representation learned via GNN's message passing is likely to possess high similarity with its neighbors' representations. Consequently, uncertainty quantification based on the learned representations will assign the same individual from the non-overlapping area with a low uncertainty (i.e., high prediction confidence), which is misleading and thus undesirable. To mitigate the feature collapse, we constrain our GraphDKL model with Lipschitzness \cite{behrmann2019invertible} to preserve the local distances in the latent space, such that the semantic manifold of the original variables is preserved in every intermediate latent space during the sequential, layer-by-layer neural mapping. Hence, predictions on high- and low-confidence samples can be effectively distinguished, making it possible for uncertainty-aware causal effect estimation on networked observational data.




\vspace{-0.2cm}
\section{Related Work}

\textbf{Deterministic Model For Causal Effect Estimation}. So far, various neural methods \cite{shalit2017estimating,shi2019adapting,zheng2023dream} have been proposed due to the proliferation of deep learning (DL). These parametric models are good at modelling the individual-level causal effect and are applicable to unseen instances. Up to date, causal effect estimation has been extended to the graph domain \cite{guo2020learning}, where the rich relational information is utilized to learn more robust deconfounded latent representations. However, all the above-mentioned models are deterministic, which can result in over-confident estimations \cite{wang2021rethinking} and is incapable of quantifying the prediction uncertainty to inform the causal estimation failure when the positivity assumption is violated.

\textbf{Probabilistic Model For Causal Effect Estimation}. It is noted that some attention has been paid to quantifying the predictive uncertainty in causal effect estimation with non-graph data. For example, the light-weight models BART \cite{chipman2010bart} and CMGP \cite{alaa2017bayesian} can offer predictive uncertainty for causal effect estimation, but they lack strong expressive power and fail to capture the complex relationship when modelling the high-dimensional data. To fix this issue, \cite{jesson2020identifying} and \cite{zhang2020learning} leverage deep Bayesian methods to enhance the expressive power and become more capable than BART and CMGP. However, little attention has been paid to estimating the causal effect on network data with uncertainty.

\section{Preliminaries}

We aim to estimate individual treatment effect (ITE) on the networked data $(\{\textbf{x}_{i}, t_{i}, y_{i}\}_{i=1}^{N}, \mathbf{A})$, where $\textbf{x}_i$, $t_i$, $y_i$ are respectively the raw variables, observed treatment, treatment outcome that correspond to the $i$-th individual, and $\mathbf{A}\in \{0,1\}^{N\times N}$ is the adjacency matrix indicating the connections between individuals, which can be obtained via consanguinity, doctor referrals, social networks, etc. The common practice is to learn a deconfounded latent representation $\textbf{z}_i$ for each individual with a GNN by aggregating its neighbour information~\cite{guo2020learning}, which is then used for counterfactual ITE prediction. To achieve this, three common assumptions are needed to lay the theoretical foundation.

\begin{assumption}[Stable Unit Treatment Value Assumption (SUTVA)]
For any individual $i$: (a) the potential outcomes for $i$ do not vary with the treatment assigned to other individuals; and (b) there are no different forms or versions of each treatment that may lead to different potential outcomes.
\end{assumption}

\begin{assumption}[Unconfoundedness]
Treatment assignment is independent to the potential outcomes $\{Y_{t=0}, Y_{t=1}\}$ given the latent covariate \textbf{z}, i.e., $t\indep \{Y_{t=0}, Y_{t=1}\} | \textbf{z}$. Note that the potential outcomes $Y$ use a different notation w.r.t. the observed ones $y$.
\end{assumption}

\begin{assumption}[Positivity]
For every $\textbf{z}$, the treatment assignment mechanism obeys: $0<p(t=1|\textbf{z})<1$. 
\end{assumption}

\subsection{Hurdles in Quantifying Uncertainty with Graph Data}
Based on the latent representation $\textbf{z}$ of each individual, we aim to estimate its treatment effect and assign an uncertainty to this estimation. Unfortunately, uncertainty quantification can be seriously poisoned by the feature collapse issue \cite{van2021feature}, especially for latent features extracted by deep neural networks (DNNs). Feature collapse describes the scenario where two distinct points in the original feature space $\mathcal{X}$ can be mapped to two similar or even identical positions in the latent space $\mathcal{Z}$. Consequently, predictions on non-overlapping samples could be incorrectly assigned an uncertainty as low as predictions on overlapping ones due to their collapsed representations. 

Despite the popularity of GNNs in learning individual representations for ITE estimation, little attention has been paid to the potential feature collapse issue. So far, the state-of-the-art GNN backbones, e.g, GraphSAGE \cite{hamilton2017inductive}, rely on non-linear mapping and are hence vulnerable to feature collapse. The message-passing scheme in GNNs potentially deteriorates the uniqueness of learned representations even further. For a proof-of-concept, we generate a toy graph with two-dimensional node features and four classes, as shown in Figure \ref{fig:feature_collapse} (a). We train a 1-layer GraphSAGE with classes 0, 1 and 2, and nodes from class 3 are held out. The two-dimensional visualization in Figure \ref{fig:feature_collapse} (b) shows that the representations generated by the trained GraphSAGE for class 3 nodes mostly collapse with class 0 in the latent space. Such collapse is different from the over-smoothing problem with GNNs, as only a shallow 1-layer structure is used and the node representations for the first three classes do not collide.

\begin{figure}[t!]
  \centering

  \subfigure[]{\includegraphics[width=0.22\textwidth]{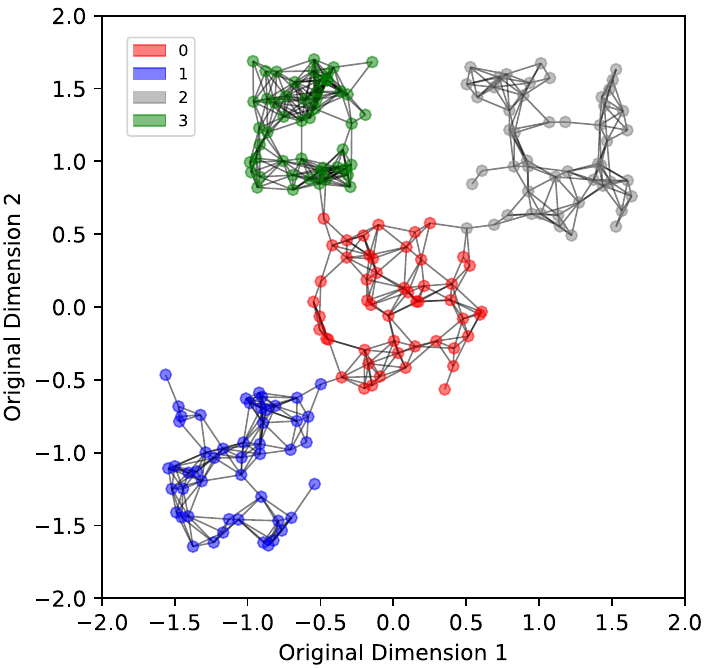}}
  \subfigure[]{\includegraphics[width=0.21\textwidth]{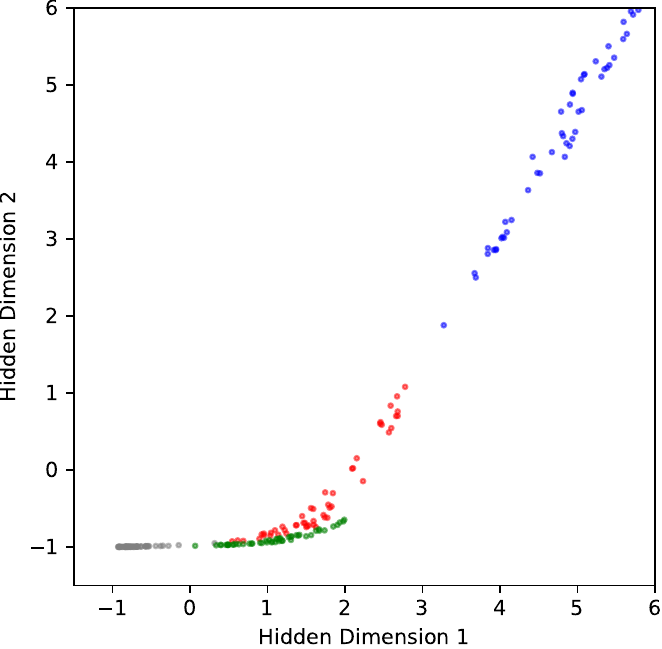}}
\vspace{-0.4cm}
  \caption{(a) A toy example graph with four classes; (b) Latent representation from a 1-Layer GraphSAGE.}
  \label{fig:feature_collapse}
\vspace{-0.5cm}
\end{figure}

\section{Methodology}

In this section, we present our graph deep kernel learning (GraphDKL) framework for handling the causal effect estimation with uncertainty on graph-structured data.

\subsection{Lipschitz-constrained Graph Representation Learning}
As a versatile framework, GraphDKL is agnostic to any GNNs. Without loss of generality, we leverage GraphSAGE \cite{hamilton2017inductive} as the base GNN given its balance between efficiency and effectiveness, and the ability to scale with batch training. At the $l$-th layer of GraphSAGE, the core neural operation to learn the latent representation for individual/node $i$ is:
\begin{equation}
\textbf{h}_{i}^{l} = \sigma(\textbf{W}_{l}\cdot\text{MEAN}(\{\textbf{h}_{i}^{l-1}\}\cup\{\textbf{h}_{j}^{l-1}, \forall j\in \mathcal{N}(i)\})),
\label{eq:sage_conv}
\end{equation} where $\sigma$ is the non-linear activation, $\textbf{W}_{l}$ is the weight matrix at layer $l$, while the mean aggregator MEAN($\cdot$) is used to merge the representations of node $i$ and its neighbours $j\in \mathcal{N}(v)$ from layer $l-1$.

To facilitate uncertainty estimation, our proposed framework, shown in Figure \ref{fig:graphdkl}, combines the GNN with deep kernel-based GP to get the best of both worlds -- the deconfounded node representations containing both individual features and structural information are extracted via GNN first, and the learned representations are fed into two independent DNNs, with each of them mapping the graph-based representations to treated and control latent spaces for subsequent predictions. For notation simplicity, we omit the formulation of each DNN, which is a multi-layer perceptron (MLP) with $L'$ layers, and takes the final-layer representation $\textbf{h}_{i}^{L}$ from GNN as its input. Unless specified, the following descriptions on DNNs apply to both treatment branches $t\in\{0,1\}$.

\textbf{Decoupling Collapsed Representations}. To alleviate feature collapse and ensure accurate uncertainty estimation, we propose to preserve the local distance among points after non-linear mapping. In GraphDKL, this constraint needs to be enforced in both the GNN for learning individual representations, as well as the two DNN branches that respectively model treated and control groups. 
In a nutshell, the distance $||\textbf{x}_i -\textbf{x}_j||$ between any two points $\textbf{x}_i$ and $\textbf{x}_j$ from the raw feature space has a corresponding meaningful distance in the latent space. To achieve this desired property, we introduce the notion of Lipschitz constant. Specifically, for each given function $\textbf{s}' = f(\textbf{s})$ with  input $\textbf{s}$ and output $\textbf{s}'$, then the Lipschitz constant $\text{Lip}(f)$ w.r.t. $f(\cdot)$ satisfies that, for any pair of inputs $(\textbf{s}_1, \textbf{s}_2)$,  $||\textbf{s}_1'-\textbf{s}_2'||\leq \text{Lip}(f)||\textbf{s}_1 -\textbf{s}_2||$ holds. In other words, $\text{Lip}(f) \geq \frac{||\textbf{s}_1' -\textbf{s}_2'||}{||\textbf{s}_1-\textbf{s}_2||}$ for any $(\textbf{s}_1, \textbf{s}_2)$ pair. If $\text{Lip}(f)\leq 1$, then it essentially means that the difference in function values is controlled by the original pairwise distance obtained from the input space. This property ensures that small changes in the input result in small changes in the output, providing a sense of stability and predictability, and $f(\cdot)$ is also termed 1-Lipschitz (local distance preserving). With the context given, we define the 1-Lipschitz GraphDKL below.


\begin{figure}[t!]
\centerline{\includegraphics[scale=0.55]{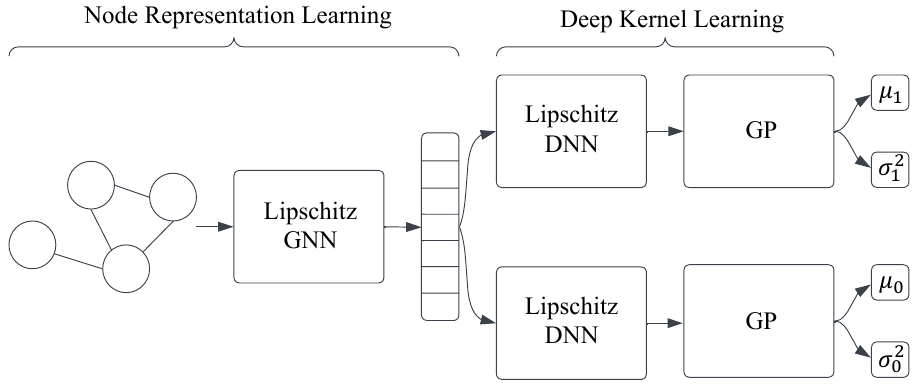}}
\caption{Structure of GraphDKL framework. The Lipschitz prefix denotes the Lipschitz-constrained neural networks.}
\label{fig:graphdkl}
\vspace{-0.5cm}
\end{figure}

\begin{theorem}[1-Lipschitz GraphDKL]
    GraphDKL has $L$ layers of graph convolution $\textbf{H}^L= g_{L}(g_{L-1}(...g_1(\mathbf{X},\mathbf{A})))$ in the GNN where $\mathbf{X}\in \mathbb{R}^{N\times D}$, $\mathbf{A}\in \{0,1\}^{N\times N}$, and $\mathbf{H}^L\in \mathbb{R}^{N\times S}$ respectively denote the $D$-dimensional raw variables, adjacency matrix, and $S$-dimensional latent representations of $N$ individuals. The GNN is followed by an $L'$-layer DNN $\textbf{Z}^{L'}=\phi_{L'}(\phi_{L'-1}(...\phi_{1}(\textbf{H}^L)))$ in either the treated/control branch, with $\textbf{Z}^{L'}\in \mathbb{R}^{N\times S}$ being the $N$ representations from the final layer. The entire representation learning pipeline in GraphDKL is 1-Lipschitz if: 
    \begin{equation}
    \left\{\begin{split}
        &\text{Lip}(g_{l})\,\,\leq1, \quad\forall l\,\,\leq L\\ 
        &\text{Lip}(\phi_{l'})\leq1, \quad\forall l'\leq L'
    \end{split}\right.,
    \end{equation}
\end{theorem} where $\text{Lip}(g_{l})$ and $\text{Lip}(\phi_{l'})$ respectively denote the Lipschitz constant of a single GNN and DNN layer.
\begin{proof}
    We denote the hidden representation of individual $i$ at the $l$-th GNN layer as $\textbf{h}^{l}_{i}$, and that of the same $i$ at the $l'$-th DNN layer as $\textbf{z}^{l'}_{i}$. Note that the raw feature $\textbf{x}_i$ is the input to the first-layer GNN, whose final-layer representation $\textbf{h}^L_i$ is the input to the first-layer DNN. Then, for any pair of instances $(i,j)$, we have:
    \begin{equation}
    \begin{split}
        \frac{||\textbf{z}_{i}^{{L'}}-\textbf{z}_{j}^{{L'}} ||}{||\textbf{x}_{i}-\textbf{x}_{j}||}
        = &\frac{|| \textbf{z}_{i}^{{L'}}-\textbf{z}_{j}^{{L'}} ||}{||\textbf{z}_{i}^{L'-1}-\textbf{z}_{j}^{L'-1}||}\times \cdots \times \frac{|| \textbf{z}_{i}^{{1}}-\textbf{z}_{j}^{{1}} ||}{||\textbf{h}_{i}^{L}-\textbf{h}_{j}^{L}||}\times\\
        &\frac{||\textbf{h}_{i}^{L}-\textbf{h}_{j}^{L}||}{||\textbf{h}_{i}^{L-1}-\textbf{h}_{j}^{L-1}||} \times \cdots \times \frac{|| \textbf{h}_{i}^{1}-\textbf{h}_{j}^{1} ||}{||\textbf{x}_{i}-\textbf{x}_{j}||}\\
        \leq& \text{Lip}(\phi_{L'})\!\!\times\!\!\cdots\!\!\times\!\!\text{Lip}(\phi_{1})\!\!\times\!\!\text{Lip}(\mathbf{g}_{L})\!\!\times\!\!\cdots\!\!\times\!\!\text{Lip}(\mathbf{g}_{1}).
    \end{split}
    \end{equation}
    As every $\text{Lip}(g_{l})$ and $\text{Lip}(\phi_{l'})$ is no larger than 1 by premise, $\frac{||\textbf{z}_{i}^{{L'}}-\textbf{z}_{j}^{{L'}} ||}{||\textbf{x}_{i}-\textbf{x}_{j}||}\leq 1$. We thus conclude that the entire neural mapping from $\textbf{x}_i$ to $\textbf{z}_i^{L'}$ in GraphDKL is 1-Lipschitz.
\end{proof}

To ensure the local distance is preserved for each neural mapping layer, \textit{spectral normalization} has been proven rigorous \cite{behrmann2019invertible} for enforcing 1-Lipschitz. Taking the weight matrix $\textbf{W}_{l}$ at the $l$-th graph convolution layer in \eqref{eq:sage_conv} as an example, the spectral normalization states:
\begin{equation}
    \text{Lip}(\mathbf{g}_{l})\leq 1, \quad \text{if}\, ||\textbf{W}_{l}||_{2}\leq 1,
\end{equation} where $||\cdot||_{2}$ denotes the spectral norm, i.e., $L_2$ matrix norm of $\textbf{W}_{l}$. Compared with using the spectral norm as a penalization term for \textit{regularization} purpose, we formulate a \textit{normalization} process that strictly bounds the spectral norm to a designated value, and this ensures obedience of 1-Lipschitz at all layers and thus benefit the measurement of uncertainty. 

As $||\textbf{W}_{l}||_{2}$ corresponds to the largest singular value of matrix $\textbf{W}_{l}$ which is known to be time-consuming to compute exactly, we perform power iteration \cite{gouk2021regularisation} over $\textbf{W}_{l}$ to obtain an approximation $\tau$ of the spectral norm, which is a lower bound on the largest singular value $||\textbf{W}_{l}||_{2}$. Then, the weight matrix is normalized as:
\begin{equation}
\overline{\textbf{W}}_{l}=
 \frac{1}{\tau}\textbf{W}_{l},
\end{equation} which empirically makes $\text{Lip}(g_{l})\leq1$ consistent across all scenarios by rescaling $\textbf{W}_{l}$ \cite{behrmann2019invertible}. Analogously, the same spectral normalization is adopted on all DNN layers' weight matrices. 

\subsection{Deep Kernel Learning}
In each treatment effect prediction branch, the deep kernel learning (DKL) module passes the latent representation $\textbf{z}^{L'}_i$ from the final DNN layer into a Gaussian process (GP) for causal effect estimation with uncertainty. From now on, we let $\textbf{z}_i=\textbf{z}^{L'}_i$ for better clarity when there is no ambiguity. 

\textbf{Standard GP}. A standard GP is a finite number of random variables which have a joint Gaussian distribution \cite{williams2006gaussian}. Mathematically, it is denoted as $\mathcal{GP}$, with mean function $\mu(\cdot)\!:\!\mathcal{X}\!\rightarrow\!\mathbb{R}$ and covariance function $k(\cdot,\cdot)\!:\!\mathcal{X}\!\times\!\mathcal{X}\!\rightarrow\!\mathbb{R}$ over the real-valued stochastic function $f(\cdot)\!:\!\mathcal{X}\!\rightarrow\!\mathbb{R}$ whose input is the $D$-dimensional variable vector $\textbf{x}\in\mathbb{R}^{D}$, namely, 
\begin{equation}
    f(\textbf{x})\sim\mathcal{GP}(\mu(\textbf{x}), k(\textbf{x},\textbf{x}')).
    \label{eq:function}
\end{equation}
By evaluating the GP at $N$ samples $\textbf{Z}=\{\textbf{z}_{i}\}_{i=1}^{N}$ (any subset from domain $\mathcal{X}$), we end up with $N$ multivariate Gaussian distributions $\textbf{f}=\{\textbf{f}_{i}\}_{i=1}^{N}$ as follows:
\begin{equation}\label{eq:f}
    \textbf{f}\sim\mathcal{N}(\bm{\mu}, \textbf{K}),
\end{equation} 
where $\bm{\mu}\in \mathbb{R}^N$ is the variance vector and $\textbf{K}\in \mathbb{R}^{N\times N}$ is the covariance matrix. With $i,j$ for indexing, $\bm{\mu}[i]=\mu(\textbf{x}_{i})$ is $i$'s mean, $\textbf{K}[i,j]=k(\textbf{x}_{i},\textbf{x}_{j})$ is the covariance between $i$ and $j$. 

\textbf{Ramping up Expressiveness}. Given the limited capacity of standard GP in learning the latent distributions \cite{damianou2013deep,wilson2016deep}, recent frameworks have been expanding the expressiveness of the standard GP. For instance, deep Gaussian process \cite{damianou2013deep} stacks a series of GPs, and deep kernel \cite{wilson2016deep} utilizes the latent variables produced from a deep learning method for the GP. 
In this paper, we investigate the deep kernel framework since it is a natural extension to our neural architecture, as well as its superiority in expressiveness and computational efficiency. Specifically, by replacing the raw variables $\textbf{x}$ with the latent output $\textbf{z}$ from GraphDKL's neural mapping, \eqref{eq:f} is updated with $\bm{\mu}[i]=\mu(\textbf{z}_i)$ and $\textbf{K}[i,j]=k(\textbf{z}_{i},\textbf{z}_{j})$.

In GP, a mean of zero is normally assumed, i.e., $\bm{\mu}=\textbf{0}$, and we consider the infinitely smooth radial basis function as the kernel for computing the covariance, i.e., $k(\textbf{z}_{i},\textbf{z}_{j})=\sigma_{ker}^{2}\exp{(-\frac{(\textbf{z}_{i}-\textbf{z}_{j})^{2}}{2l^{2}})}$, where $\boldsymbol{\theta}_{ker}=\{\sigma_{ker}, l\}$ is a parameter set of the GP kernel to be optimized. With $N$ latent representations $\textbf{Z}\in \mathbb{R}^{N\times S}$ learned from $\textbf{X}$ and their corresponding real-valued labels $\textbf{y}=\{y_{i}\}_{i=1}^{N}$, we can obtain the joint marginal likelihood w.r.t. the updated \eqref{eq:f} as follows:
\begin{equation}
    p(\textbf{y}|\textbf{Z},\boldsymbol{\theta}_{ker}) = \int \prod_{i=1}^{N} p(y_{i}|\textbf{f}_{i},\textbf{z}_{i})p(\textbf{f}_{i}|\textbf{z}_{i})d\textbf{f} = \mathcal{N}(\textbf{y}|\textbf{0}, \textbf{K}),
    \label{eq:exact}
\end{equation} where the optimal kernel parameters are obtained by finding its maximum through gradients.

\textbf{ITE Prediction with Uncertainty Quantification.} To evaluate the model at an arbitrary test point $\textbf{x}_{*} \in \mathbb{R}^{D}$ with corresponding latent representation $\textbf{z}_{*}$, we leverage the property that the joint distribution of the training labels \textbf{y} and the test label $y_{*}$ is still Gaussian:
\begin{equation}
\begin{pmatrix}
y_{*}\\
\textbf{y}
\end{pmatrix}\sim\mathcal{N}\begin{pmatrix}
\textbf{0}
,& \begin{bmatrix}
k(\textbf{z}_{*},\textbf{z}_{*})& \textbf{k}^{\text{T}}_{*}\\
\textbf{k}_{*}& \textbf{K}
\end{bmatrix}
\end{pmatrix},
\label{eq:joint_Gaussian}
\end{equation} where column vector $\textbf{k}_{*} \in \mathbb{R}^N$ measures the covariance between $\textbf{z}_{*}$ and all $N$ training samples, i.e., $\textbf{k}_{*}[i]=k(\textbf{z}_{*},\textbf{z}_{i})$. Thus, the posterior label distribution is:
\begin{equation}
y_{*}|(\textbf{z}_{*}, \textbf{Z}, \textbf{y}, \boldsymbol{\theta}_{ker})\sim\mathcal{N}(\mu_{*}, \sigma_{*}^{2}),
\label{eq:conditional}
\end{equation}
which has the following closed-form solution \cite{williams2006gaussian}:
\begin{equation}
    \mu_{*}=\textbf{k}_{*}^{\text{T}}\textbf{K}^{-1}\textbf{y},\,\,\,
    \sigma_{*}^{2} = k(\textbf{z}_{*},\textbf{z}_{*})-\textbf{k}_{*}^{\text{T}}\textbf{K}^{-1}\textbf{k}_{*},
\end{equation}
 where the mean $\mu_{*}$ serves as the prediction of $y_*$ w.r.t. a treatment $t\in\{0,1\}$, and the variance $\sigma_{*}^{2}$ of the prediction is used as a direct indicator of the uncertainty w.r.t. sample $\textbf{x}_*$. Essentially, the value of $\sigma_{*}^{2}$ holds a positive correlation with the uncertainty.

\subsection{Sparse Variational Optimization}

If the exact Gaussian process were applied to our proposed GraphDKL framework, the model would suffer from a $\mathcal{O}(N^{3})$ complexity due to the inversion of the covariance matrix $\textbf{K}\in\mathbb{R}^{N\times N}$, which is computationally prohibitive when handling large graphs. To increase scalability, we adopt a sparse GP \cite{williams2006gaussian} with stochastic variational inference (SVI) \cite{salimbeni2017doubly} to our setting, building a computationally tractable GraphDKL framework with the ability to scale better. 

We start by assuming a set of latent inducing points $\textbf{M}=\{\textbf{m}_{i}\}_{i=1}^{M}$ in the same latent space as $\textbf{Z}$. For the stochastic function $f(\cdot)$ with Gaussian prior in \eqref{eq:function}, we obtain the corresponding outputs $\textbf{v}=f(\textbf{Z})$ and $\textbf{u}=f(\textbf{M})$ w.r.t. the distributions of $\textbf{Z}$ and $\textbf{M}$ in the same space, respectively.  To form a tractable objective, we derive the evidence lower bound $\mathcal{L}$ as follows:
\vspace{-0.3cm}
\begin{equation}
\begin{split}
    \log p(\textbf{y}) &\!= \log\! \int\! p(\textbf{y},\textbf{v},\textbf{u})d\textbf{v}d\textbf{u}
    = \log \!\int\! \frac{p(\textbf{y},\textbf{v},\textbf{u})}{q(\textbf{v},\textbf{u})}q(\textbf{v},\textbf{u})d\textbf{v}d\textbf{u}\\
    &\!=\log\mathbb{E}_{q(\textbf{v},\textbf{u})}\!\!\left[\frac{p(\textbf{y},\textbf{v},\textbf{u})}{q(\textbf{v},\textbf{u})}\right]\!\!\geq\!\!\mathcal{L}=\mathbb{E}_{q(\textbf{v},\textbf{u})}\!\!\left[\log\frac{p(\textbf{y},\textbf{v},\textbf{u})}{q(\textbf{v},\textbf{u})}\right],
\end{split}
\label{eq:ELBO}
\end{equation}
where the SVI process appproximates the posterior $q(\textbf{v},\textbf{u})$ by minimizing the Kullback-Leibler divergence $\text{KL}(q||p)$ between the variational posterior $q$ and the prior $p$ \cite{salimbeni2017doubly}. 
As we essentially aim to perform SVI with a set of global variables, we let $\textbf{u}$ take this role with a variational distribution $q(\textbf{u})$, and follow the widely accepted variational posterior \cite{salimbeni2017doubly} $q(\textbf{v},\textbf{u}) = p(\textbf{v}|\textbf{u})q(\textbf{u})$. We set $q(\textbf{u})\sim \mathcal{N}(\textbf{u}|\bm{\mu}_u, \textbf{K}_{u})$ with mean $\bm{\mu}_u\in\mathbb{R}^{M}$ and covariance $ \textbf{K}_{u}\in\mathbb{R}^{M\times M}$ to be learned.
To this end, the evidence lower bound (ELBO) $\mathcal{L}$ can be further decomposed into the following form, with an additional constraint on $q(\textbf{u})$: 
\begin{equation}
    \mathcal{L}=\mathbb{E}_{q(\textbf{v})}\left[\log p(\textbf{y}|\textbf{v})\right] - \text{KL}(q(\textbf{u})|| p(\textbf{u})),
    \label{eq:simplified_ELBO}
\end{equation} where $p(\textbf{u})$ is a priori. Since posterior $q(\textbf{u})$ is Gaussian, with the analytically achievable $p(\textbf{v}|\textbf{u})$ analogous to \eqref{eq:joint_Gaussian} and \eqref{eq:conditional} by conditioning on the prior $p(\textbf{u})$, the variational posterior $q(\textbf{v})$ can be analytically obtained as follows:
\begin{equation}
    q(\textbf{v})=\int p(\textbf{v}|\textbf{u})q(\textbf{u})d\textbf{u}=\mathcal{N}(\textbf{v}|\Tilde{\bm{\mu}}, \Tilde{\Sigma}),
\end{equation} where $\Tilde{\bm{\mu}}$ and $ \Tilde{\Sigma}$ are parameterized w.r.t. $\bm{\mu}_u$ and $\textbf{K}_{u}$. Therefore, the first term in the simplified 
 $\mathcal{L}$ can be calculated with Monte Carlo sampling since posterior $q(\textbf{v})$ is available with the unknown parameters $\{\bm{\mu}_u,\textbf{K}_{u}\}$ to be learned. Notably, the time complexity of sparse variational GraphDKL is significantly reduced due mainly to the smaller $M\times M$ matrix to be inverted, bringing a non-dominant $\mathcal{O}(M^{3})$ complexity. Consider the matrix multiplication in deriving $\Tilde{\Sigma}$ in \ref{eq:simplified_ELBO} for the $q(\textbf{v})$ to sample $N$ times in order to calculate $\mathcal{L}$, the asymptotic time complexity is capped to $\mathcal{O}(M^{2}N)$ \cite{williams2006gaussian,salimbeni2017doubly}. With $M\ll N$ in our case, handling large graph-structured data with GraphDKL is tractable. 

By optimizing the tractable objective $\mathcal{L}$ derived by SVI, we solve the unknown parameters such that the variational posterior $q(\textbf{v})$, which only depends on the input $\textbf{x}$, can be optimized to fit the data. Eventually, with the fully trained GraphDKL, given a test point $\textbf{x}_{*}$ (with latent representation $\textbf{z}_*$), we can use the much smaller $\textbf{K}'\in \mathbb{R}^{M\times M}$ given by priori, $\textbf{z}_*$'s covariance with those $M$ prior samples $\textbf{k}_*'\in\mathbb{R}^{M}$, and the optimized posterior $q(\textbf{u})$ to obtain prediction (i.e., mean) $\mu_{*}$ and variance (i.e., uncertainty) $\sigma^2_{*}$:  
\begin{equation}
\mu_{*}=\Gamma\textbf{u}_{u},\,\,\,
\sigma_{*}^{2}= k(\textbf{z}_*,\textbf{z}_*)- \Gamma (\textbf{K}'-\textbf{K}_u)\Gamma^{\text{T}},
\label{eq:sparseGP_variance}
\end{equation} 
where $\Gamma={\textbf{k}_*'}^{\text{T}}{\textbf{K}'}^{-1}$.

To conclude, the final ITE estimation and its associated prediction uncertainty w.r.t. test sample $\textbf{x}_{*}$ have the following approximations:
\begin{equation}
\begin{split}
    \!\!\!\text{ITE}_{*} &=\mathbb{E}[Y_{t=1}\!-\!Y_{t=0}|\textbf{z}_{*}] \simeq \mu_{*,t=1} \!-\! \mu_{*,t=0},\\
    \!\!\!\text{Uncertainty}_{*} &=\mathbb{E}[(Y_{t=1}\!-\!Y_{t=0})^{2}|\textbf{z}_{*}] \simeq \sigma_{*,t=1}^{2} \!+\! \sigma_{*,t=0}^{2}.
\end{split}
\end{equation}

\section{Experiments}
We present our experimental analysis in this section.

\begin{table*}[t!]
    \caption{The proportion in each column represents the fixed percentage of the test samples rejected by each method. Thus, we calculate the $\sqrt{\epsilon_{\text{PEHE}}}$ over the same-size retained test samples by averaging results from 10 simulations for each setting.}
    \vspace{-0.1cm}
    \centering
    \renewcommand{\arraystretch}{0.85}
    \setlength\tabcolsep{4pt}
    \begin{adjustbox}{scale=0.95}
    \begin{tabular}{|P{0.1cm}|>{\raggedright\arraybackslash}P{1.3cm}|>{\centering\arraybackslash}P{0.5cm}|>{\centering\arraybackslash}P{0.5cm}|>{\centering\arraybackslash}P{0.5cm}|>{\centering\arraybackslash}P{0.5cm}|>{\centering\arraybackslash}P{0.5cm}|>{\centering\arraybackslash}P{0.5cm}|>{\centering\arraybackslash}P{0.5cm}|>{\centering\arraybackslash}P{0.5cm}|>{\centering\arraybackslash}P{0.5cm}|>{\centering\arraybackslash}P{0.5cm}||>{\centering\arraybackslash}P{0.5cm}|>{\centering\arraybackslash}P{0.5cm}|>{\centering\arraybackslash}P{0.5cm}|>{\centering\arraybackslash}P{0.5cm}|>{\centering\arraybackslash}P{0.5cm}|>{\centering\arraybackslash}P{0.5cm}|>{\centering\arraybackslash}P{0.5cm}|>{\centering\arraybackslash}P{0.5cm}|>{\centering\arraybackslash}P{0.5cm}|>{\centering\arraybackslash}P{0.5cm}|}

 \hline

 \multicolumn{2}{|c|}{Dataset} & \multicolumn{10}{c||}{BlogCatalog} & \multicolumn{10}{c|}{Flickr} \\
 \hline
 $k$&Method& 0\%& 5\%& 10\%& 15\%& 20\%& 25\%& 30\%& 50\%& 70\%& 90\%&0\%& 5\%& 10\%& 15\%& 20\%& 25\%& 30\%& 50\%& 70\%& 90\%\\
 \hline
 
 \multirow{6}*{\rotatebox{90}{0.5}}
&BART       &10.15  &10.18  &10.21  &10.23  &10.28  &10.30  &10.31  &10.16  &9.73   &9.27 & 8.86&8.86&8.86&8.86&8.87&8.87&8.88&8.89&8.90&8.85\\
&BCFRMMD    &7.93   &6.96   &6.24   &5.53   &5.04   &4.56   &4.25   &3.72   &3.66   &3.83 & 48.65&44.86&43.01&41.67&40.82&39.60&38.97&35.02&28.40&3.47\\
&BCEVAE     &42.76  &37.7   &33.29  &29.06  &25.59  &23.07  &21.67  &20.79  &21.25  &21.38 & 53.81&38.13&33.54&31.18&29.26&28.05&27.41&25.43&22.54&9.60\\
&CMGP       &10.89  &10.18  &9.41   &9.31   &9.32   &9.15   &8.99   &8.73   &9.14   &9.33 &10.37&5.86&4.97&4.38&4.05&3.77&3.63&3.28&2.99&2.86\\
&GraphDKL &\textbf{4.31}&4.21&\textbf{3.98}&\textbf{3.80}&\textbf{3.67}&\textbf{3.48}&\textbf{3.31}&\textbf{2.90}&\textbf{2.64}&\textbf{2.21}& \textbf{3.92}&\textbf{3.24}&\textbf{3.10}&\textbf{3.01}&\textbf{2.95}&\textbf{2.92}&\textbf{2.92}&\textbf{2.84}&\textbf{2.69}&\textbf{2.46}\\
 \hline

 \hline
\multirow{6}*{\rotatebox{90}{1}}
&BART       &12.97   &12.89   &12.94   &12.96   &13.01   &13.07   &13.15   &13.00   &12.16   &11.63   & 15.70&15.70&15.7&15.7&15.62&15.63&15.64&15.46&15.41&15.51\\
&BCFRMMD    &10.45   &9.29    &8.58    &7.93    &7.37    &6.90    &6.46    &6.13    &6.33    &6.69    & 10.87&7.59&6.12&5.20&4.59&4.14&\textbf{3.81}&\textbf{3.37}&\textbf{3.40}&\textbf{3.61}\\
&BCEVAE     &36.75   &33.11   &29.97   &27.51   &24.93   &22.99   &21.85   &20.83   &21.40   &21.82   & 24.63&16.09&12.99&11.36&10.67&10.33&10.11&9.91&10.03&10.06\\
&CMGP       &11.46   &9.84    &9.18    &8.69    &8.44    &8.01    &7.70    &7.06    &6.73    &6.76    & 18.51&10.15&8.28&7.23&6.58&6.17&5.87&5.28&5.15&6.41\\
&GraphDKL & \textbf{5.08}&\textbf{4.79}&\textbf{4.48}&\textbf{4.32}&\textbf{4.17}&\textbf{4.00}&\textbf{3.90}&\textbf{3.76}&\textbf{3.63}&\textbf{2.97}&\textbf{7.29}&\textbf{4.49}&\textbf{4.23}&\textbf{4.11}&\textbf{4.04}&\textbf{3.97}&3.93&3.91&3.91&3.71\\
 \hline

 \hline
 \multirow{6}*{\rotatebox{90}{2}}
&BART       &30.96   &31.32   &31.51   &31.65   &31.97   &31.63   &31.48   &31.78   &30.10   &29.50   & 28.80&28.80&28.81&28.75&28.76&28.76&28.73&28.53&28.8&28.38\\
&BCFRMMD    &26.83   &23.31   &21.09   &19.11   &17.55   &16.12   &15.06   &14.15   &14.56   &15.14   & 18.80&13.20&10.59&9.05&7.95&7.12&6.50&5.71&5.83&6.22\\
&BCEVAE     &42.76   &37.70   &33.29   &29.06   &25.59   &23.07   &21.67   &20.79   &21.25   &21.38   & 20.95&13.76&11.31&10.09&9.49&9.20&8.96&8.55&8.60&8.82\\
&CMGP       &28.20   &25.00   &23.87   &23.06   &22.70   &22.32   &22.10   &21.19   &21.39   &23.20   & 34.66&20.24&17.05&14.84&13.54&12.77&12.48&12.91&15.09&23.36\\
&GraphDKL & \textbf{10.32}&\textbf{8.92}&\textbf{8.13}&\textbf{7.56}&\textbf{7.23}&\textbf{6.82}&\textbf{6.64}&\textbf{5.96}&\textbf{5.49}&\textbf{5.33}& \textbf{12.38}&\textbf{6.73}&\textbf{6.26}&\textbf{6.00}&\textbf{5.85}&\textbf{5.76}&\textbf{5.71}&\textbf{5.67}&\textbf{5.58}&\textbf{5.38}\\

\hline

\end{tabular}
    
\end{adjustbox}
\label{tab:blogcatalog}
\vspace{-0.4cm}
\end{table*}

\begin{figure*}[t!]

  \centering
  \begin{minipage}[b]{0.75\textwidth}
  \centering
  \includegraphics[scale=0.35]{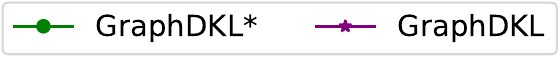}\\
  \vspace{-0.0cm}
  \subfigure[{\fontsize{6.5}{7}\selectfont BlogCatalog-0.5}]{\includegraphics[width=0.17\textwidth]{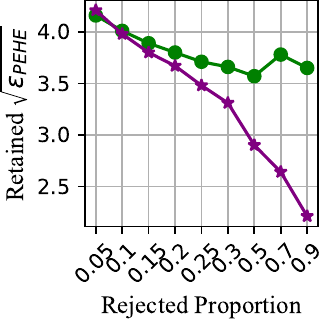}}
  \subfigure[{\fontsize{6.5}{7}\selectfont BlogCatalog-0.5}]{\includegraphics[width=0.15\textwidth]{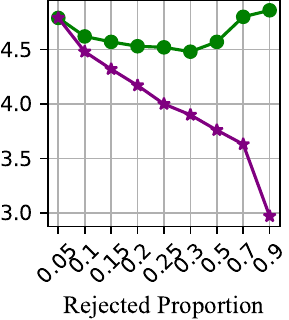}}
  \subfigure[{\fontsize{6.5}{7}\selectfont BlogCatalog-2}]{\includegraphics[width=0.14\textwidth]{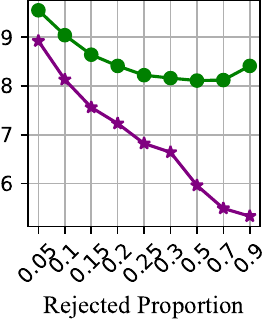}}
  \subfigure[{\fontsize{6.5}{7}\selectfont Flickr-0.5}]{\includegraphics[width=0.15\textwidth]{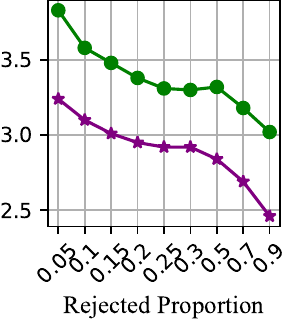}}
  \subfigure[{\fontsize{6.5}{7}\selectfont Flickr-1}]{\includegraphics[width=0.15\textwidth]{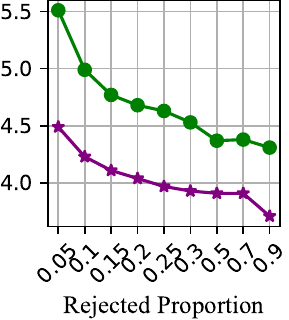}}
  \subfigure[{\fontsize{6.5}{7}\selectfont Flickr-2}]{\includegraphics[width=0.146\textwidth]{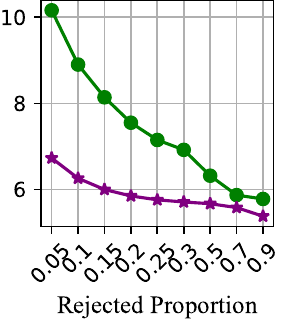}}
  \caption{Ablation study on Lipschitz constraint of the neural mapping. Lipschitz-constrained GraphDKL has a clear performance gain over GraphDKL* without such constraint.}
  \label{fig:Lipschitz_macro}
  \end{minipage}
  \hspace{0cm}
  \begin{minipage}[b]{0.20\textwidth}
  \centering
    \includegraphics[scale=0.4]{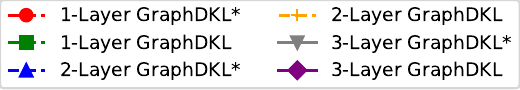}\\
    \includegraphics[width=0.7\textwidth]{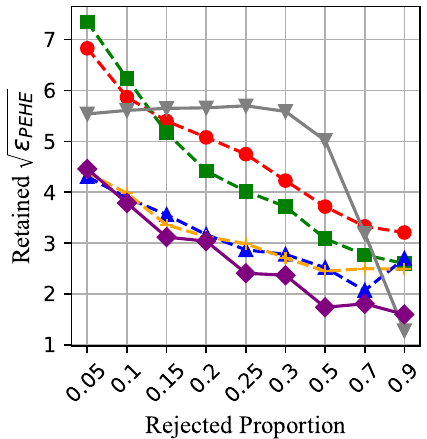}
    \caption{Rejection performance with different GraphDKL variants.}
    \label{fig:Lipschitz_micro}
  \end{minipage}

\vspace{-0.5cm}
\end{figure*}

\subsection{Experimental Setup}

\subsubsection{Dataset} 
We adopt two public benchmarks with networked observational data: \textbf{Blogcatalog \cite{guo2020learning}}  and \textbf{Flickr \cite{guo2020learning}}. Both BlogCatalog and Flickr datasets are processed and simulated in the same practice in \cite{guo2020learning}. For both datasets, three settings are created with $k$ = 0.5, 1, and 2, respectively, where $k$ denotes the magnitude of the imbalance in the semi-synthetic dataset. The higher the $k$ value, the more imbalanced the dataset is. In total, we evaluate the model performance in six different scenarios.

\subsubsection{Metric}
We use precision in estimation of heterogeneous effect (PEHE) \cite{shalit2017estimating}, a well-established metric defined as $\sqrt{\epsilon_{\text{PEHE}}}=\sqrt{\Sigma_{i=1}^{N}((Y_{i,t=1}-Y_{i,t=0})-(\mu_{i,t=1}-\mu_{i,t=0}))^{2}/N}$ for measuring the treatment estimation accuracy at the individual level. The lower the $\sqrt{\epsilon_{\text{PEHE}}}$, the better the performance.

\subsubsection{Baselines} Note, that our proposed GraphDKL is the first model to handle causal effect estimation with uncertainty on graph data. To obtain better comparisons, we share the learned node representation with the other baselines which can only be operated on the non-graph data: BART \cite{chipman2010bart}, BCFRMMD \cite{jesson2020identifying}, BCEVAE \cite{jesson2020identifying}, and CMGP~\cite{alaa2017bayesian}. 

\subsubsection{Evaluation Scheme on Uncertainty Quantification}

We randomly split each dataset into train/val/test with a 3/1/1 ratio. To evaluate the most effective uncertainty-aware method (a.k.a. rejection method), we reject the estimations with the highest uncertainty and calculated the $\sqrt{\epsilon_{\text{PEHE}}}$ over the retained samples: the lower the retained $\sqrt{\epsilon_{\text{PEHE}}}$, the better the rejection method. As setting an uncertainty threshold for rejection can be domain-specific in real-world cases, here we use the specific uncertainty threshold given by each method that rejects a certain proportion of the top most-uncertain test samples. We test on an increasing proportion $\{0\%, 5\%, 10\%, 15\%, 20\%, 25\%, 30\%, 50\%, 70\%, 90\%\}$ in our experiments, where the $\sqrt{\epsilon_{\text{PEHE}}}$ scores are reported for all methods from the same amount of retained test samples.

\subsection{Rejection Policy Performance}

Since the main task of this paper is to explore the pivot of positivity assumption with the uncertainty-aware model for causal effect estimation on graph data. We compare our proposed  GraphDKL with various rejection methods with the main results shown in Table \ref{tab:blogcatalog}. When compared to other rejection methods, our method GraphDKL always initializes with a lower $\sqrt{\epsilon_{\text{PEHE}}}$ at 0\% rejection rate, even though all the other baselines designed for independent data leverage the same learned node representations from the GraphSAGE convolution. Furthermore, GraphDKL outperforms all the other models over the retained test set in terms of the following key performance: (1) it keeps rejecting the bad estimation while preserving the lowest $\sqrt{\epsilon_{\text{PEHE}}}$ on both datasets under most settings; (2) it has the fastest error convergence with an increased rejection rate.

\subsection{Ablation Study on Lipschitzness}

We conduct a detailed ablation study on the Lipschitz constraint. We use GraphDKL and GraphDKL* to respectively denote variants with and without this constraint. As shown in \ref{fig:Lipschitz_macro}, GraphDKL is superior to GraphDKL* across all the scenarios. Note, that the base model's performance on BlogCatalog datasets has a bouncing-back retained $\sqrt{\epsilon_{\text{PEHE}}}$ when rejecting more samples on the test set. Additionally, we compare its influence to GNNs with varying capacity by using different graph convolution layers. 
Based on the results in Figure \ref{fig:Lipschitz_micro}, the Lipschitz-constrained 3-Layer GraphDKL has the best rejection performance as shown in Figure \ref{fig:Lipschitz_micro} by decoupling the collapsed representation to get more accurate uncertainty of each estimation, while the proposed spectral norm can effectively bring performance gain for different GNN variants in uncertainty-aware counterfactual prediction.

\vspace{-0.1cm}
\section{Conclusion}
We investigate the violation of the positivity assumption for causal effect estimation on graph data and take a novel perspective to create a safer causal estimator on graph data -- quantifying the estimation uncertainty. Extensive experiments on the two widely used semi-synthetic graph datasets show the superiority of our proposed Lipschitz GraphDKL over the other baselines in terms of identifying high-risk estimations.


\vspace{-0.1cm}
\section{Acknowledgement}
This work is supported by the Australian Research Council under the streams of Future Fellowship (No. FT210100624), Discovery Project (No. DP190101985), Discovery Early Career Researcher Award (No. DE230101033), and Industrial Transformation Training Centre (No. IC200100022). 

\vspace{-0.1cm}
\bibliographystyle{IEEEtran}
\bibliography{graphdkl}

\end{document}